\documentclass{article}

\usepackage{iclr2027_conference,times}

\usepackage{amsmath,amsfonts,bm}

\def\eqref#1{equation~\ref{#1}}

\def\1{\bm{1}}

\DeclareMathAlphabet{\mathsfit}{\encodingdefault}{\sfdefault}{m}{sl}
\SetMathAlphabet{\mathsfit}{bold}{\encodingdefault}{\sfdefault}{bx}{n}

\usepackage{hyperref}
\hypersetup{pdfborder={0 0 1.5}}
\usepackage{url}
\usepackage{graphicx}
\usepackage{booktabs}
\usepackage{caption}
\usepackage{array}
\usepackage{float}
\usepackage{amsmath,amssymb,amsthm}

\newtheorem{theorem}{Theorem}
\newtheorem{lemma}{Lemma}

\theoremstyle{definition}
\newtheorem{definition}{Definition}
\theoremstyle{remark}
\newtheorem{remark}{Remark}

\title{Kolmogorov--Arnold Networks:\\
Curse of Dimensionality and Structural Bias}

\author{%
\begin{tabular}{c}
\textbf{Yunhao Ni$^{1}$, Yuxin Guo$^{1}$, Jie Luo$^{1}$, Wenjun Wu$^{1,2}$, Lei Huang$^{1,2}$}\thanks{Corresponding author: \texttt{huangleiAI@buaa.edu.cn}.}\\[0.7em]
{\normalfont $^{1}$State Key Laboratory of Complex and Critical Software Environment,}\\
{\normalfont Beihang University, Beijing, China}\\[0.2em]
{\normalfont $^{2}$Hangzhou International Innovation Institute, Beihang University, Hangzhou, China}
\end{tabular}}

\iclrfinalcopy

\begin{document}
\raggedbottom

\maketitle
\lhead{Preprint}

\begin{abstract}
Kolmogorov--Arnold networks (KANs) have a dimension-independent approximation
exponent when the target admits a sufficiently smooth representation by
univariate edge functions. We ask whether this favorable scaling belongs to
the architecture or to the assumed target class. For arbitrary-depth standard
spline-KANs, we prove a worst-case lower bound on isotropic Sobolev balls:
below a dimension-independent accuracy threshold, the required number of
spline coefficients grows at least exponentially with input
dimension. Existing composition-matched MLP results, by contrast, show that
favorable exponents on low-arity targets are not exclusive to KANs. We also
revisit empirical comparisons. On four scaling-law targets, spline-KAN
approximants constructed from known decompositions transfer exactly to
MLPs with the cubic ReLU activation $\sigma(x)=(\max\{0,x\})^3$, although
their full dense parameter counts are larger. In controlled ridge-profile scans, KANs
have a strong advantage when each nonlinear component depends on one input
coordinate. When a component couples multiple input coordinates, this
advantage may narrow sharply. At tight budgets, interacting targets often
favor MLPs. In the
random-label test on MNIST, ReLU-MLPs fit the randomized training labels more efficiently than
spline-only KANs. Together, these
results locate favorable scaling in target structure and a compatible
inductive bias, not a general KAN-specific escape from the curse of
dimensionality.

\end{abstract}

\section{Introduction}

Kolmogorov--Arnold networks (KANs) were introduced as an alternative to
standard multilayer perceptrons (MLPs)~\citep{rumelhart1986learning}, replacing fixed node-wise activations
with learnable univariate edge functions, typically parameterized by
splines~\citep{liu2024kan}. This parameterization makes the learned univariate
transformations directly inspectable and distinguishes KANs from the standard
node-activation design of MLPs.

A particularly striking result reported by \citet{liu2024kan} is that KANs can
beat the curse of dimensionality (CoD) on compositional targets with smooth
Kolmogorov--Arnold (KA) representations \citep{liu2024kan}. In
their approximation analysis, beating CoD means that the exponent in the
approximation rate does not deteriorate with the input dimension. Assuming
that the target admits a KAN representation with sufficiently smooth
univariate edge functions, they approximate those functions by degree-$k$
splines and obtain the dimension-independent exponent $k+1$. To support this prediction
empirically, they construct five toy targets known in advance to admit smooth
KAN representations, train KANs and MLPs under the same vanilla LBFGS
protocol, and report that the KAN curves approach the predicted $P^{-4}$ rate ($P$ denotes parameter count)
while the tested MLPs scale more slowly and plateau
\citep{liu2024kan}. They interpret this contrast as
KANs exploiting an intrinsically low-dimensional representation that MLPs do
not similarly exploit~\citep{liu2024kan}.

Different from the discussion above, We argue that the phenomenon "KANs escape CoD" because of \textit{the relative ease
of approximating the chosen target classes, rather than the uniform
approximation advantage of KANs.}
First, KANs suffer from CoD on harder classes. Sobolev
balls are a canonical approximation benchmark
\citep{devore1989optimal,mhaskar1996optimal,yarotsky2017error}.
\citet{liu2024kan} acknowledge that nonlinear $n$-width theory
rules out a general escape from CoD on Sobolev spaces, but do not
quantify the corresponding worst-case parameter cost for standard spline-KANs.
We address this gap by proving that arbitrary-depth standard spline-KANs suffer from worst-case
CoD on this canonical class: below a dimension-independent accuracy threshold,
the required number of spline coefficients grows at least
exponentially with the input dimension $d$
(\mbox{Theorem~\ref{thm:kan_cod}}). Our proof uses the exact representation of
spline-KANs by $\operatorname{ReLU}^k$ MLPs from
\citet[Theorem~3.3]{wang2024expressiveness}, together with an explicit count
of the representing MLP's affine parameters, to transfer a Sobolev lower bound.
Second, low-arity compositional classes
are easy enough for MLPs to beat CoD as well. Long before KANs,
composition-matched MLP theory showed that sparse MLPs attain approximation
exponents governed by local arity rather than the input dimension
\citep{mhaskar2016deep,mhaskar2017when}. We discuss the corresponding target
classes and approximation rates in detail in
Section~\ref{sec:composition-matched}. Together, these observations suggest
that the favorable dimension dependence reflects the structure of the target
class rather than a general approximation advantage of KANs.

The lower bound leaves open how the two architectures compare on particular
targets at finite budgets. We first revisit four functions from the
scaling-law experiment of \citet{liu2024kan}: when their
decompositions are supplied, constructed spline-KAN approximants transfer to
ReLU$^3$ MLPs with essentially the same test error, albeit at potentially
larger dense parameter counts (Section~\ref{sec:constructive-scaling}). We
then train both families on controlled ridge profiles. KANs have a pronounced
advantage when each nonlinear component uses just one input coordinate.
Even limited coupling across input coordinates sharply narrows it, and the
finite-budget ranking also depends on the target nonlinearity. See
Section~\ref{sec:ridge-boundary} for details. Finally, a random-label test on MNIST
shows that ReLU-MLPs fit the randomized training labels more efficiently
than spline-only KANs
(Section~\ref{sec:mnist-random-label}).

Our contributions are threefold: (i) we prove a dimension-uniform,
fixed-accuracy CoD lower bound for arbitrary-depth standard spline-KANs on
isotropic Sobolev classes; (ii) we place conditional KAN rates alongside
composition-matched MLP rates, identifying target structure as the source of
their favorable dimensional dependence; and (iii) we use constructive scaling
comparisons, controlled ridge targets, and MNIST random-label fitting to distinguish
representational possibilities from finite-budget training behavior.

\section{Related Work}
\label{sec:related-work}

\subsection{Favorable rates under structural assumptions}
\label{sec:composition-matched}

\citet{kolmogorov1957representation} proved that continuous multivariate
functions admit representations through continuous univariate functions and
addition. \citet{schmidt2021kolmogorov} showed that the outer univariate
functions in such representations can be highly irregular even when the
target is smooth. Motivated by the representation theorem,
\citet{liu2024kan} introduce KANs and obtain an approximation-error exponent
independent of the input dimension when the target admits a sufficiently
smooth Kolmogorov--Arnold representation. They describe this conditional rate
as beating the curse of dimensionality (CoD), although its
representation-dependent constant may depend on the dimension.
Separately, \citet[Theorem~3.2]{wang2024expressiveness} show that degree-one
spline-KANs can exactly realize ReLU-MLPs on bounded domains, providing a
route to universality without assuming a smooth Kolmogorov--Arnold
representation.
For residual KANs, \citet{kratsios2025besov} establish approximation rates in
Besov norms. For spline-KAN regression estimators,
\citet{liu2025rate} derive a dimension-free risk exponent conditional on a KAN
representation with smooth univariate components.

Years before KANs were introduced, \citet{mhaskar2016deep} showed how matching
MLP connectivity to low-arity target structure improves approximation
complexity. A concrete
binary-tree class is studied by \citet{mhaskar2017when}: for integers $d\ge2$
and $n\ge1$, let $\mathcal{C}^{(2)}_{d,n}$ consist of functions represented
by a tree with $d$ input leaves and $d-1$ internal nodes. Each node computes a
bivariate constituent in $W^{n,\infty}([-1,1]^2)$, the Sobolev space with
essentially bounded weak derivatives through order $n$. The Sobolev and
Lipschitz norms are uniformly bounded. \citet{mhaskar2017when} show that a
network matching this tree can approximate every
$f\in\mathcal{C}^{(2)}_{d,n}$ uniformly to accuracy $\epsilon\in(0,1)$ using
$N$ active units, with
\begin{equation}
N=O\!\left((d-1)\epsilon^{-2/n}\right)
\label{eq:compositional-mlp-rate}
\end{equation}
The exponent $2/n$ reflects the local arity, while $d$ enters through the
number of constituent nodes. \citet{mhaskar2020function} extend the
compositional analysis to directed-acyclic-graph classes under compatible
architecture, smoothness, and error-propagation assumptions.
Like KANs, MLPs can attain approximation exponents governed by local arity
rather than the ambient dimension on composition-matched classes. Favorable
dimensional scaling on such targets is not an exclusive property of KANs.

\subsection{Approximation limits on broad function classes}

Approximation rates quantify the model size required at a prescribed accuracy
beyond universal approximation. For $d$-variate Sobolev balls of smoothness
order $n$, \citet{devore1989optimal} establish nonlinear-width limits.
For networks with suitable smooth activations, \citet{mhaskar1996optimal}
obtain $\epsilon^{-d/n}$-type approximation complexity at uniform accuracy
$\epsilon$. \citet{yarotsky2017error} derive upper and lower bounds for deep
ReLU networks under their own architectural assumptions. These results
retain explicit worst-case dimension dependence on broad smoothness classes.

For KANs, \citet[Theorem~3.2 and Corollary~3.4]{wang2024expressiveness}
transfer deep piecewise-polynomial
MLP constructions and obtain the Sobolev approximation rate
$O(P^{-2n/d})$, where $P$ is the number of trainable parameters. This upper
bound has an exponent depending on $d$, but it cannot establish that this
dependence is unavoidable: a better KAN rate is not ruled out. A worst-case
lower bound is needed to prove a CoD requirement for the architecture.

On the lower-bound side, \citet{shen2022optimal} provide a H\"older-class
shattering argument. Combining it with a Vapnik--Chervonenkis
(VC)-dimension bound, \citet{hu2026approximation} derive a worst-case
uniform-error lower bound for two-layer B-spline KANs of prescribed
complexity. It implies a CoD-type parameter requirement at fixed accuracy.

Our result addresses arbitrary-depth standard spline-KANs on integer-order
Sobolev balls. Whereas \citet{hu2026approximation} bound the VC dimension of
two-layer KANs directly, we construct Sobolev bumps, use the converse
KAN-to-MLP representation of \citet[Theorem~3.3]{wang2024expressiveness},
derive the dense MLP parameter count needed here, and apply a VC-dimension
bound to the transferred networks. The target classes, architectures, and
transfer steps make the two lower bounds complementary.

\subsection{Subsequent KAN developments and limitations}

Subsequent work has broadened the KAN family: \citet{aghaei2024rkan} propose
rational edge functions, and \citet{bozorgasl2024wavkan} introduce wavelet
KANs. \citet{qiu2024powermlp} introduce PowerMLP as an MLP-type alternative
and establish a structured KAN-to-PowerMLP representation under a
nonzero-parameter count (Theorem~2). Direct analyses of
the original spline-KAN's theoretical or empirical
limitations are comparatively less prominent. Under matched
parameter and computation budgets, \citet{yu2024fairer} find that MLPs usually
win outside symbolic formula representation and trace KAN's advantage there
largely to its B-spline basis. In scientific machine learning,
\citet{shukla2024fair} find B-spline KANs less accurate and efficient than
MLP counterparts on their tested problems, while polynomial variants remain
sensitive to seed and order. These comparisons test practical performance
under particular tasks and protocols. Our controlled ridge-profile and MNIST
experiments further probe how target structure affects the finite-budget
ordering of KANs and MLPs.

\section{Preliminaries}

Our lower bound transfers from
$\operatorname{ReLU}^k$-MLPs to spline-KANs. This section defines the target
class and both model families, then states the edge-level spline realization
used in Section~\ref{sec:main-results}.
Throughout, $O_k(\cdot)$ hides a multiplicative constant depending only on $k$.

For a multi-index
$\boldsymbol{\alpha}=(\alpha_1,\ldots,\alpha_d)\in\mathbb N_0^d$, where
$\mathbb N_0:=\{0,1,2,\ldots\}$, set
$|\boldsymbol{\alpha}|:=\sum_{j=1}^d\alpha_j$ and
$D^{\boldsymbol{\alpha}}f:=
\partial^{|\boldsymbol{\alpha}|}f/
(\partial x_1^{\alpha_1}\cdots\partial x_d^{\alpha_d})$.
Let $C^n([0,1]^d)$ denote the space of functions with continuous derivatives
through order $n$, and, for $h:[0,1]^d\to\mathbb R$, write
$\|h\|_\infty:=\sup_{\mathbf x\in[0,1]^d}|h(\mathbf x)|$. For positive integers $d,n$, let
\begin{equation}
F_{d,n}:=
\left\{
f\in C^n([0,1]^d):
\max_{|\boldsymbol{\alpha}|\le n}
\|D^{\boldsymbol{\alpha}}f\|_{\infty}\le 1
\right\}.
\end{equation}
The complexity question is how many trainable parameters are needed to
approximate every member of $F_{d,n}$ to a prescribed uniform accuracy.
The architecture is fixed, but its parameter values may depend on the target
function.

\subsection{\texorpdfstring{$\operatorname{ReLU}^k$}{ReLU-k}-MLPs}

\paragraph{Multilayer perceptron.}
Let $d_0=d_{\mathrm{in}}$ and $d_L=d_{\mathrm{out}}$ be the input and output
dimensions. An $L$-layer MLP with
scalar activation $\sigma:\mathbb R\to\mathbb R$ realizes
\begin{equation}
f=T_L\circ\sigma\circ T_{L-1}\circ\cdots\circ\sigma\circ T_1,
\qquad
T_\ell(x)=W_\ell x+b_\ell,
\end{equation}
where $W_\ell\in\mathbb R^{d_\ell\times d_{\ell-1}}$,
$b_\ell\in\mathbb R^{d_\ell}$, and $\sigma$ acts coordinate-wise. Thus $L$
counts affine layers. Its trainable-parameter count is
\begin{equation}
P_{\mathrm{MLP}}=\sum_{\ell=1}^{L}d_\ell(d_{\ell-1}+1).
\end{equation}
For an integer $k\ge1$, write $(x)_+:=\max\{x,0\}$ and
$\sigma_k(x):=(x)_+^k$. An MLP with activation
$\sigma_k$ is called a $\operatorname{ReLU}^k$-MLP. The superscript denotes
an ordinary power.
The activation $\sigma_k$ is also called a truncated power: it is zero for
$x\le0$ and equals $x^k$ for $x>0$.

\subsection{Spline-KANs}
\label{sec:prelim-spline-kans}

\citet{liu2024kan} implement KANs with learnable edge functions based on
B-splines. Each layer sums the edge outputs at its receiving nodes.
We first introduce B-spline basis and edge function, and then define the standard Spline-KAN version studied here.

\paragraph{B-spline basis.}
Let $k$ and $G$ be positive integers, and let
$\mathcal T=(\tau_1,\ldots,\tau_{G+2k+1})$ be a strictly increasing augmented
knot sequence whose central knots form a grid of $G$ intervals. The associated
degree-$k$ B-spline basis is denoted by
$\{B_{r,k}\}_{r=1}^{G+k}$. Our indexing and the Cox--de Boor recursion are
given in Appendix~\ref{app:model-details}.

\paragraph{$(k,G)$-spline edge function.}
For the B-spline basis above, a spline edge is
\begin{equation}
\phi(x)=\sum_{r=1}^{G+k}c_rB_{r,k}(x),
\end{equation}
where the coefficients $c_r\in\mathbb R$ are trainable.

\begin{definition}[Standard spline-KAN]
\label{def:standard-spline-kan}
Given positive integers $d_0,\ldots,d_L$, a spline degree $k$, and a grid size
$G$, a depth-$L$ standard spline-KAN with shape $[d_0,\ldots,d_L]$ is defined
for $x\in\mathbb R^{d_0}$ recursively by $x^{(0)}=x$ and
\begin{equation}
\label{eq:standard-kan-layer}
x^{(\ell+1)}_j
=
\sum_{i=1}^{d_\ell}
\phi^{(\ell)}_{j,i}\!\left(x^{(\ell)}_i\right),
\qquad j=1,\ldots,d_{\ell+1},
\end{equation}
where $x^{(\ell)}\in\mathbb R^{d_\ell}$, every
$\phi^{(\ell)}_{j,i}$ is a $(k,G)$-spline edge, and the output is
$g(x)=x^{(L)}$. Its width is $w:=\max_{0\le\ell\le L}d_\ell$.
\end{definition}

Each edge contributes $G+k$ trainable spline coefficients, so
Definition~\ref{def:standard-spline-kan} gives
\begin{equation}
\label{eq:kan-param-count}
P_{\mathrm{KAN}}
=(G+k)\sum_{\ell=0}^{L-1}d_\ell d_{\ell+1}
=O\!\left(Lw^2(G+k)\right).
\end{equation}
This agrees with the usual count for equally wide spline-KANs
\citep{liu2024kan}.

\paragraph{Model boundary.}
The implementation of \citet{liu2024kan} augments each spline edge with a
weighted sigmoid linear unit (SiLU) base branch and permits grid-update
heuristics. A grid
update changes the knot locations while preserving the final degree-$k$,
$G$-interval spline form. The representation result below therefore allows
arbitrary knot locations that may differ across edges. We exclude the
additional edge-wise SiLU term and non-spline KAN
variants. Appendix~\ref{app:model-details} gives the precise scope.

\subsection{Exact spline realization by
\texorpdfstring{$\operatorname{ReLU}^k$}{ReLU-k}}

The classical B-spline identity expresses each degree-$k$ B-spline as a
linear combination of at most $k+2$ shifted copies of $\sigma_k$
\citep{deBoor1978}. \citet[Lemma~1]{qiu2024powermlp} also use this identity in a
power-activation network construction. We record the edge-level form needed
below.

\begin{lemma}[Exact realization of a spline edge by $\operatorname{ReLU}^k$]
\label{lem:spline-edge-repu}
Every $(k,G)$-spline edge function has an exact representation
\begin{equation}
\phi(x)=\sum_{m=1}^{M}\beta_m\sigma_k(x-\tau_m),
\qquad M\le(k+2)(G+k)=O_k(G+k),
\end{equation}
where the shifts are knots of the edge and the coefficients $\beta_m$ depend
linearly on the spline coefficients of $\phi$. Hence the edge is realized by
a one-hidden-layer $\operatorname{ReLU}^k$ subnetwork of width $O_k(G+k)$.
\end{lemma}

We use this edge-level form in the next section to bound the full affine
parameter count of the ordinary MLP representation.

\section{Worst-Case CoD of Spline-KANs}
\label{sec:main-results}

We now study the worst-case approximation complexity of arbitrary-depth
standard spline-KANs on $F_{d,n}$. The argument has three steps: an exact
representation of spline-KANs by MLPs with a controlled parameter count
(Lemma~\ref{lem:spline-kan-repu}); a Sobolev approximation lower bound for
$\operatorname{ReLU}^k$-MLPs (Lemma~\ref{lem:repu_vc}); and a transfer of that
lower bound back to spline-KANs (Theorem~\ref{thm:kan_cod}).

\subsection{Representation bridge}

\citet[Theorem~3.3]{wang2024expressiveness} give an exact representation of
spline-only KANs by ordinary $\operatorname{ReLU}^k$-MLPs, allowing knot
locations to differ across edges. Our lower bound requires the full affine
parameter count of an ordinary MLP. We therefore state \citet{wang2024expressiveness}'s
representation result and then derive that count for the model of
Definition~\ref{def:standard-spline-kan}.

\begin{lemma}[Spline-KAN representation by $\operatorname{ReLU}^k$-MLPs, \citealp{wang2024expressiveness}]
\label{lem:wang-spline-kan-repu}
Suppose that $g:[0,1]^d\to\mathbb R$ is represented by a depth-$L$
spline-only KAN of width $w$ and grid size $G$, whose edge functions are
degree-$k$ B-splines. The knot locations may differ across edges.
Then there exists a $\operatorname{ReLU}^k$-MLP $f$ of width
$(G+2k+1)w^2$ and depth at most $2L$ such that
\begin{equation}
f(x)=g(x)
\end{equation}
for all $x\in[0,1]^d$.
\end{lemma}

\begin{lemma}[Parameter count for the representation bridge]
\label{lem:spline-kan-repu}
Let $g$ be a depth-$L$ standard spline-KAN with shape
$[d_0,\ldots,d_L]$, width $w$, $(k,G)$-spline edges, and parameter count
$P_{\mathrm{KAN}}$ as in \eqref{eq:kan-param-count}. The edge-wise knot
sequences may be arbitrary and may differ across edges. Then there exists a
$\operatorname{ReLU}^k$-MLP $f$ that represents $g$ exactly. If
$W_{\mathrm{hid}}$, $L_{\mathrm{MLP}}$, and $P_{\mathrm{MLP}}$ denote its
maximum hidden width, depth, and total number of affine weights and biases,
respectively, then there is a constant $C_k>0$, depending only on $k$, such
that
\begin{equation}
W_{\mathrm{hid}}
\le
C_k w^2(G+k),
\qquad
L_{\mathrm{MLP}}
\le
C_k L,
\end{equation}
and
\begin{equation}
P_{\mathrm{MLP}}
\le
C_k P_{\mathrm{KAN}}^2.
\end{equation}
In particular, the coarser width-based parameter bound
\begin{equation}
P_{\mathrm{MLP}}
\le
C_k Lw^4(G+k)^2
\end{equation}
also holds.
\end{lemma}

\noindent We provide the proof of Lemma~\ref{lem:spline-kan-repu} in
Appendix~\ref{app:proof-spline-kan-repu}.

\begin{remark}[Full versus nonzero parameter counts]
Lemma~\ref{lem:spline-kan-repu} counts every affine weight and bias of the
ordinary $\operatorname{ReLU}^k$-MLP. If one instead counts only nonzero
parameters, \citet[Theorem~2]{qiu2024powermlp} give a structured
KAN-to-PowerMLP representation that preserves this count under their
convention. This result counts all the parameters, for the
VC-dimension argument below uses the full MLP parameter count.
\end{remark}

\subsection{A Sobolev lower bound for \texorpdfstring{$\operatorname{ReLU}^k$}{ReLU-k}-MLPs}

The next lemma adapts the shattering argument of
\citet[Theorem~4]{yarotsky2017error}. We construct Sobolev bumps and combine
the resulting VC-dimension lower bound with an upper bound for fixed-degree
piecewise-polynomial MLPs~\citep{anthony2009neural,bartlett1998vc}.

\begin{lemma}[Sobolev approximation lower bounds for $\operatorname{ReLU}^k$-MLPs]
	\label{lem:repu_vc}
	Fix \(d,n,k\) with \(d,n\ge 1\), \(k\ge 1\), and let \(W\) and \(L\) denote, respectively, the number of trainable parameters and affine layers in the fixed MLP architecture. Approximation error is measured by the uniform norm \(\|\cdot\|_\infty\) on \([0,1]^d\), parameter values may depend on the target function, and \(\ln\) denotes the natural logarithm.
	\begin{enumerate}
		\item[(a)] For any \(\epsilon\in(0,1)\), a $\operatorname{ReLU}^k$-MLP architecture capable of approximating every \(f\in F_{d,n}\) with error \(\epsilon\) must have at least \(c\,\epsilon^{-d/(2n)}\) trainable parameters, where \(c=c(d,n,k)>0\).
		\item[(b)] Let \(p\ge 0\), \(c_1>0\) be constants. For any \(\epsilon\in(0,\frac12)\), if a fixed $\operatorname{ReLU}^k$-MLP architecture of depth \(L\le c_1\ln^p(1/\epsilon)\) can approximate every \(f\in F_{d,n}\) with error \(\epsilon\), then the architecture must have at least
		\begin{equation}
		c_2\,\epsilon^{-d/n}\ln^{-2p-1}(1/\epsilon)
		\end{equation}
		trainable parameters, with a constant \(c_2=c_2(d,n,p,c_1,k)>0\).
	\end{enumerate}
	Moreover, there exist constants \(\epsilon_0=\epsilon_0(n)>0\) and
	\(c_0=c_0(k)>0\), both independent of \(d\), such that the lower
	bound in part~(a) has the dimension-uniform form
	\begin{equation}
	W\ge c_0
	\left(\frac{\epsilon_0}{\epsilon}\right)^{d/(2n)},
	\qquad 0<\epsilon\le\epsilon_0.
	\end{equation}
\end{lemma}

\noindent\emph{Proof sketch.}
Following \citet[Theorem~4]{yarotsky2017error}, choose a grid
resolution $N$, place $N^d$ separated points in $[0,1]^d$, and construct
disjoint smooth bumps whose
values encode arbitrary binary labels while remaining in $F_{d,n}$. Uniform
$\epsilon$-approximation of all these functions forces the thresholded network
class to shatter the $N^d$ points. Combining this VC-dimension lower bound
with the standard upper bounds for fixed-degree piecewise-polynomial networks
gives parts~(a) and~(b). The full argument, including the constants needed for
the dimension-uniform statement, is given in
Appendix~\ref{app:proof-repu-vc}.

\subsection{Standard spline-KANs suffer from CoD}

We use CoD in the worst-case approximation sense: at some fixed target
accuracy, the parameter complexity required for the entire function class
grows at least exponentially with the input dimension. The dimension-uniform
statement below makes this dependence explicit.

\begin{theorem}[Standard spline-KANs suffer from CoD]
\label{thm:kan_cod}
Let \(d,n,k\) satisfy \(d,n\ge 1\) and \(k\ge 1\), and write
\(H:=G+k\). Consider a standard spline-KAN architecture in the sense of
Definition~\ref{def:standard-spline-kan}, with input dimension \(d_0=d\),
scalar output \(d_L=1\), depth \(L\), width \(w\), grid size \(G\), and
parameter count \(P_{\mathrm{KAN}}\) as in \eqref{eq:kan-param-count}. Then the
following statements hold.
\begin{enumerate}
	\item[(a)] For any \(\epsilon\in(0,1)\), if such a standard spline-KAN architecture can approximate every \(f\in F_{d,n}\) with uniform-norm error at most \(\epsilon\), then
	\begin{equation}
	Lw^4H^2
	\ge
	c\,\epsilon^{-d/(2n)}
	\end{equation}
	and, in particular,
	\begin{equation}
	P_{\mathrm{KAN}}
	\ge
	c'\,\epsilon^{-d/(4n)}.
	\end{equation}
	Here \(c,c'>0\) depend only on \(d,n,k\).
	\item[(b)] Let \(p\ge 0\), \(c_1>0\) be constants. For any \(\epsilon\in(0,\frac12)\), if such a standard spline-KAN has depth \(L\le c_1\ln^p(1/\epsilon)\) and can approximate every \(f\in F_{d,n}\) with uniform-norm error at most \(\epsilon\), then
	\begin{equation}
	Lw^4H^2
	\ge
	c_2\,\epsilon^{-d/n}\ln^{-2p-1}(1/\epsilon)
	\end{equation}
	and, in particular,
	\begin{equation}
	P_{\mathrm{KAN}}
	\ge
	c_3\,\epsilon^{-d/(2n)}\ln^{-p-\frac12}(1/\epsilon),
	\end{equation}
	with constants \(c_2,c_3>0\) depending only on \(d,n,p,c_1,k\).
\end{enumerate}
Moreover, there exist constants \(\epsilon_0=\epsilon_0(n)>0\) and
\(\bar c,\bar c'>0\), depending on \(k\) but independent of \(d\), such
that, for \(0<\epsilon\le\epsilon_0\), part~(a) can be written as
\begin{equation}
Lw^4H^2
\ge
\bar c\left(\frac{\epsilon_0}{\epsilon}\right)^{d/(2n)},
\qquad
P_{\mathrm{KAN}}
\ge
\bar c'\left(\frac{\epsilon_0}{\epsilon}\right)^{d/(4n)}.
\end{equation}
Thus, for fixed \(\epsilon<\epsilon_0\), the parameter count grows at least exponentially in \(d\).
\end{theorem}

\noindent\emph{Proof sketch.}
Apply Lemma~\ref{lem:spline-kan-repu} to the given spline-KAN
architecture class. If it uniformly approximates $F_{d,n}$, the resulting
$\operatorname{ReLU}^k$-MLP architecture does so as well, with
$P_{\mathrm{MLP}}\le C_kLw^4(G+k)^2$ and
$P_{\mathrm{MLP}}\le C_kP_{\mathrm{KAN}}^2$. Lemma~\ref{lem:repu_vc}
then yields the two structural lower bounds and, after taking square roots,
the stated bounds for $P_{\mathrm{KAN}}$. The logarithmic-depth case follows
in the same way from Lemma~\ref{lem:repu_vc}(b). See
Appendix~\ref{app:proof-kan-cod} for the complete proof.

Our theorem shows that standard spline-KANs still suffer from worst-case CoD
on isotropic Sobolev balls. Together with the composition-matched MLP
results in Section~\ref{sec:composition-matched}, this suggests that KANs and
MLPs can have comparable approximation capabilities on structured targets.
We next examine their relative performance experimentally.

\section{Empirical Boundaries Between KANs and MLPs}
\label{sec:experiments}

We revisit the KAN--MLP scaling-law comparison by constructing KAN
approximants from known target decompositions and transferring them to MLPs.
We then compare trained models on controlled ridge targets and on MNIST with
standard or random labels. Protocols are in
Appendix~\ref{app:experimental-details}. In the ridge scans, $P$ counts trainable
parameters used in the numerical forward pass, and MSE denotes mean squared
error.

\subsection{Revisiting the KAN--MLP scaling-law comparison}
\label{sec:constructive-scaling}

\citet{liu2024kan} compare trained KANs and MLPs on structured
functions and report faster error decay for KANs. We revisit four of those
targets---$J_0(20x)$, $e^{\sin(\pi x)+y^2}$, a 100-dimensional composition,
and a four-dimensional nested composition---to distinguish what an MLP can
represent from what training finds. We omit $xy$, which our degree-three spline
construction represents up to numerical precision at every tested grid.
Using the known target decompositions, we construct spline-only KAN
approximants without gradient-based optimization and transfer them
algebraically to ReLU$^3$ MLPs. The transfer preserves the represented
function at step zero. In Figure~\ref{fig:constructive-scaling}, validation
selects this unmodified MLP for the 100-dimensional and four-dimensional
targets, and a lower-error, fine-tuned MLP for Bessel and exp-sine. The MLP
curves lie farther right because their full dense parameter counts are larger.
Further protocol details are in Appendix~\ref{app:constructive-scaling}.

\begin{figure}[H]
  \centering
  \includegraphics[width=0.98\linewidth]{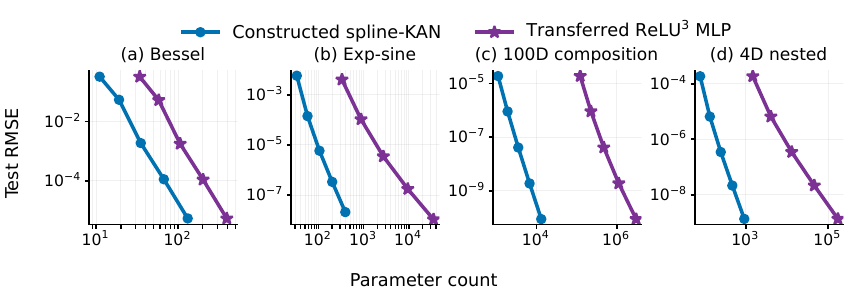}
  \caption{Constructive scaling on four targets from \citet{liu2024kan}.
  Blue: constructed degree-three spline-KAN test RMSE versus spline-coefficient
  count. Purple: transferred ReLU$^3$ MLP test RMSE versus dense parameter
  count, at the validation-selected checkpoint (step zero or up to 200 Adam
  steps). At step zero the MLP represents the same function as its KAN source.
  All panels use $G\in\{8,16,32,64,128\}$.}
  \label{fig:constructive-scaling}
\end{figure}

\subsection{Interaction support and the finite-budget boundary}
\label{sec:ridge-boundary}

For input dimension $d$, number of ridge components $R$, and per-ridge
support size $s\le d$, we consider standardized ridge-profile targets
\begin{equation}
  \widetilde f_{\varphi}(\mathbf x)=\sum_{r=1}^{R} a_r
  \varphi\!\left(\boldsymbol v_r^\top(2\mathbf x-\mathbf 1)+b_r\right),
  \qquad
  f_{\varphi}(\mathbf x)=\frac{\widetilde f_{\varphi}(\mathbf x)-\mu_{\varphi}}{\sigma_{\varphi}},
  \label{eq:ridge-profile-teacher}
\end{equation}
Here $\mathbf x$ is uniform on $[0,1]^d$, $R$ counts ridge components,
$a_r$ and $b_r$ are coefficients and offsets, and $\mathbf 1$ is the
all-ones vector. Each direction $\boldsymbol v_r$ uses exactly $s$
coordinates. We test
$\varphi(t)\in\{t^2,\sin(2\pi t),\tanh(t),\operatorname{SiLU}(t)\}$.
$\mu_{\varphi}$ and $\sigma_{\varphi}$ are estimated on an independent
reference set. At $s=1$, the target is additive and each nonlinear component
depends on one coordinate. At $s>1$, a component can couple coordinates,
introducing interactions. The coefficient, direction, and offset sampling
rules are in Appendix~\ref{app:experimental-details}.

We scan $d,R\in\{4,16,64\}$ and $s\in\{1,2,4,\ldots,d\}$, comparing degree-three
spline-only KANs with MLPs using ReLU, tanh, SiLU, or sine activation.
At each budget, each family selects its run across architectures, learning
rates, and three seeds by minimum training-probe MSE. We report its
independent test MSE. We use the square profile
$\varphi(t)=t^2$ with $d=16$ as the main-text example.

Figure~\ref{fig:ridge-square-boundary} reports
\begin{equation}
  \rho(P)=
  \frac{\operatorname{MSE}_{\text{best MLP}}(P)}
       {\operatorname{MSE}_{\text{best KAN}}(P)},
  \label{eq:ridge-profile-ratio}
\end{equation}
so $\rho>1$ favors KAN and $\rho<1$ favors MLP. On square targets at
moderate budgets, the decoupled $s=1$ targets consistently favor KAN. Once
each component involves multiple coordinates, this advantage contracts
sharply. At the tightest budgets, the interacting targets generally favor
MLPs. At the largest tested budget, every displayed curve favors KAN.
Complete results across all tested profiles and dimensions, together with
additional diagnostics, are
reported in Appendix~\ref{app:ridge-robustness}.

\begin{figure}[H]
  \centering
  \includegraphics[width=0.98\linewidth]{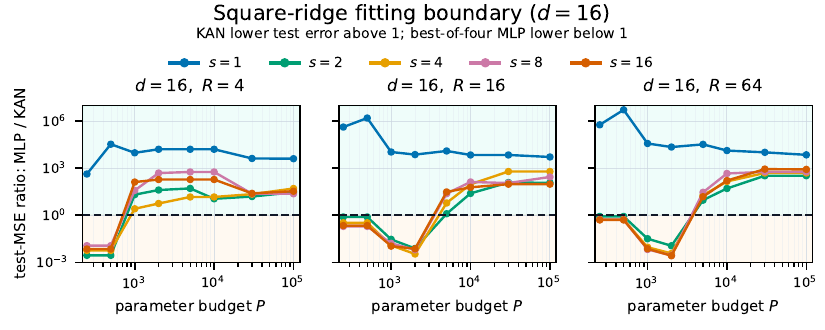}
  \caption{Training-probe-selected test-MSE ratios on square-ridge targets
  with $d=16$ and $R\in\{4,16,64\}$. Each family is selected independently
  within budget $P$. MLP candidates include four activations. The ratio
  $\rho>1$ favors KAN, and $s$ is the number of active coordinates per ridge.
  The full nine-panel scan is in Appendix~\ref{app:ridge-robustness}.}
  \label{fig:ridge-square-boundary}
\end{figure}

\subsection{MNIST classification and random-label test under matched budgets}
\label{sec:mnist-random-label}

We compare MNIST with standard labels to a random-label test in which each
training label is independently sampled from the ten classes
\citep{zhang2017rethinking}. We report test accuracy for standard labels and
final training accuracy for random labels to measure how well the models fit
arbitrary image--label assignments. Both families reach high test accuracy
with standard labels, but ReLU-MLPs do so with fewer parameters and fit
random labels far more efficiently than spline-only KANs.
The full protocol and budget frontiers are in
Appendix~\ref{app:mnist-experiment}.

\begin{table}[ht]
  \caption{MNIST accuracies (\%) with $N_{\mathrm{train}}=60{,}000$
  within selected trainable-parameter budgets.
  Standard labels report test accuracy and random labels report final training
  accuracy.
  Each entry is the best architecture median over three seeds after selecting
  optimizer and learning rate by final training accuracy. Bold marks the higher
  median within each label condition and budget.}
  \label{tab:mnist-normal-random}
  \centering
  \small
  \begin{tabular}{@{}llrrrr@{}}
    \toprule
    Labels & Model & $P\leq5\mathrm{k}$ & $P\leq20\mathrm{k}$ & $P\leq100\mathrm{k}$ & $P\leq500\mathrm{k}$ \\
    \midrule
    Standard & ReLU-MLP & \textbf{87.5} & \textbf{94.9} & \textbf{98.0} & \textbf{98.6} \\
    Standard & Spline-only KAN & 54.8 & 89.6 & 96.0 & 97.3 \\
    \midrule
    Random & ReLU-MLP & 11.7 & \textbf{18.2} & \textbf{47.6} & \textbf{100.0} \\
    Random & Spline-only KAN & \textbf{12.3} & 13.4 & 16.9 & 19.4 \\
    \bottomrule
  \end{tabular}
\end{table}

\section{Conclusion}
\label{sec:conclusion}

  Theoretically, we prove that arbitrary-depth standard spline-KANs need
  exponentially many spline coefficients to approximate isotropic Sobolev
  balls below a fixed accuracy threshold. Their architecture therefore does
  not remove worst-case CoD. Earlier composition-matched MLP results show
  that favorable dimension dependence on low-arity targets is not exclusive
  to KANs.

  Empirically, we revisit the scaling-law targets of \citet{liu2024kan}:
  transferring constructed spline-KAN approximants to MLPs separates what
  can be represented from what training discovers. On controlled ridge
  targets, KANs strongly favor decoupled components, whereas interacting
  targets often favor MLPs at tight budgets. The boundary shifts with the
  target nonlinearity and budget. The random-label test on MNIST further probes fitting
  outside these structured targets.

\textbf{Limitations.}
The theorem covers spline-only standard KANs with arbitrary edge-wise knot
locations, not variants with SiLU base branches, multiplication nodes, or
non-spline edges.

\subsection*{Acknowledgments}
This work was supported by the National Natural Science Foundation of China
(62476016, 62441617), and the Fundamental Research Funds for the Central
Universities and the Beijing Natural Science Foundation (Grant No. QY24137).

\subsection*{Code Availability}
Experiment code and compact result tables are available in the
\href{https://github.com/Musicath/kan-cod-experiments}{GitHub repository}.

\subsection*{AI use statement}
We used generative AI tools to assist with manuscript editing, proof checking,
literature search and citation checking, code implementation and debugging,
and the preparation and interpretation of synthetic experiments and figures.
The research question, mathematical claims and constructions, and experimental
objectives were formulated by the authors. All AI-assisted outputs were
reviewed and verified by the authors, who take full responsibility for the
final content.

\bibliography{arxiv/main}
\bibliographystyle{iclr2027_conference}
\newpage

\appendix
\section{Additional Model Details}
\label{app:model-details}

\subsection{B-spline convention}

For the augmented knot sequence
$\mathcal T=(\tau_1,\ldots,\tau_{G+2k+1})$ introduced in
Section~\ref{sec:prelim-spline-kans}, the central knots
$\xi_j:=\tau_{k+1+j}$, $j=0,\ldots,G$, form the $G$ grid intervals. The
remaining $k$ knots on each side are boundary or extension knots. Starting
from $B_{r,0}(x)=\mathbf 1_{[\tau_r,\tau_{r+1})}(x)$, where $\mathbf 1_A$
denotes the indicator of a set $A$, and using the usual
right-endpoint convention, the Cox--de Boor recursion is
\begin{equation}
B_{r,m}(x)
=\frac{x-\tau_r}{\tau_{r+m}-\tau_r}B_{r,m-1}(x)
+\frac{\tau_{r+m+1}-x}{\tau_{r+m+1}-\tau_{r+1}}B_{r+1,m-1}(x),
\qquad m=1,\ldots,k.
\end{equation}
Thus $B_{r,k}$ is a degree-$k$ piecewise polynomial supported on
$[\tau_r,\tau_{r+k+1}]$ and is $C^{k-1}$ across simple interior knots.

\subsection{Scope relative to practical and non-spline KANs}

\citet{liu2024kan} parameterize a practical KAN edge function as
\begin{equation}
\phi(x)=w_b b(x)+w_s\,\operatorname{spline}(x),
\qquad b(x)=\operatorname{SiLU}(x):=\frac{x}{1+e^{-x}},
\end{equation}
where $w_b$ and $w_s$ are trainable scalar weights. The base function is
introduced to improve optimization behavior, and
$\operatorname{spline}(x)$ denotes the B-spline expansion on that edge. It
forms an edge-wise additive
branch, distinct from a layer-wise residual or skip connection, and adds the
trainable scalars $w_b,w_s$. This SiLU branch is outside the
spline-only model analyzed in the main text. By contrast, a grid update only
relocates knots: after the update, each edge is still a degree-$k$ spline with
$G$ intervals, so arbitrary edge-wise knot locations are covered by
Definition~\ref{def:standard-spline-kan} and the representation bridge.

The same boundary excludes Fourier KANs \citep{zhang2025kolmogorov}, rational
KANs \citep{aghaei2024rkan}, wavelet-based KANs
\citep{bozorgasl2024wavkan,seydi2024unveiling}, and polynomial-basis variants
\citep{sidharth2024chebyshev,seydi2024exploring}. These models require
separate representation theorems, VC-dimension estimates, or approximation
lower-bound arguments beyond the spline-only scope of this paper.

\section{Proofs and Supporting Results}
\label{app:proofs}

\subsection{Proof of the parameter-count lemma}
\label{app:proof-spline-kan-repu}
\begin{proof}[Proof of Lemma~\ref{lem:spline-kan-repu}]
Set \(H:=G+k\) and \(a_\ell:=d_\ell d_{\ell+1}\). By Lemma~\ref{lem:spline-edge-repu}, each KAN layer \(\ell\) can be realized by \(M_\ell\le C_k a_\ell H\) shifted \(\operatorname{ReLU}^k\) units, followed by an affine summation that forms \(x^{(\ell+1)}\). To write the stacked construction as an ordinary MLP, use the affine map that produces the shifted \(\operatorname{ReLU}^k\) arguments for the first KAN layer, then compose the affine summation of layer \(\ell-1\) with the affine shift map of layer \(\ell\) for \(\ell=1,\ldots,L-1\), and finally use the affine summation of the last KAN layer. This gives an exact \(\operatorname{ReLU}^k\)-MLP realization of \(g\).

The hidden widths are \(M_\ell\le C_kd_\ell d_{\ell+1}H\le C_kw^2H\), and the number of nonlinear stages is \(L\). Hence \(W_{\mathrm{hid}}\le C_kw^2(G+k)\) and \(L_{\mathrm{MLP}}\le C_kL\), after adjusting the constant.

For the parameter count, the first affine map has at most \(M_0(d_0+1)\) parameters, the intermediate affine maps have at most \(M_\ell(M_{\ell-1}+1)\) parameters for \(\ell=1,\ldots,L-1\), and the final affine map has at most \(d_L(M_{L-1}+1)\) parameters. Therefore
\[
\begin{aligned}
P_{\mathrm{MLP}}
&\le C_kH^2\biggl(a_0^2
+\sum_{\ell=1}^{L-1}a_{\ell-1}a_\ell
+a_{L-1}^2\biggr)\\
&\le C_kH^2\biggl(\sum_{\ell=0}^{L-1}a_\ell\biggr)^2
= C_kP_{\mathrm{KAN}}^2.
\end{aligned}
\]
Bounding each summand in the first parenthesis by $w^4$ gives the coarser bound
\[
P_{\mathrm{MLP}}
\le
C_kLw^4H^2
=
C_kLw^4(G+k)^2.
\]
\end{proof}

\subsection{Proof of the \texorpdfstring{$\operatorname{ReLU}^k$}{ReLU-k} VC-dimension lower bound}
\label{app:proof-repu-vc}
\begin{proof}[Proof of Lemma~\ref{lem:repu_vc}]
We follow \citet[Theorem~4]{yarotsky2017error}. Let \(\mathcal H\) denote the
binary function class obtained by thresholding the outputs of the fixed
architecture while varying its \(W\) trainable parameters. Since
\(\sigma_k(x)=(x)_+^k\) is piecewise polynomial with two pieces and
fixed degree \(k\), \citet[Theorems~8.7 and~8.8]{anthony2009neural}
(see also \citealp{bartlett1998vc}) give
\begin{equation}
\label{eq:vcdim-w2}
\mathrm{VCdim}(\mathcal H)\le c_3W^2
\end{equation}
and
\begin{equation}
\label{eq:vcdim-depth}
\mathrm{VCdim}(\mathcal H)
\le \widetilde c_3L^2W\ln(2W),
\end{equation}
where \(c_3,\widetilde c_3>0\) depend only on \(k\). These are the only
activation-specific estimates used below.

Given a positive integer \(N\), let
\(S=\{\mathbf{x}_1,\ldots,\mathbf{x}_{N^d}\}\) be the regular grid of
cube centers in \([0,1]^d\), so distinct points have
\(\ell_\infty\)-distance
$\|\mathbf x-\mathbf x'\|_\infty:=\max_j|x_j-x'_j|$ at least \(1/N\).
To make the dependence on
\(d\) explicit, fix \(\psi\in C^\infty(\mathbb R)\) with
\(0\le\psi\le1\), \(\psi(0)=1\), and \(\psi(t)=0\) for
\(|t|\ge1/2\), and set
\begin{equation}
\phi_d(\mathbf{x})=\prod_{j=1}^d\psi(x_j),
\qquad
c_{4,n}:=
\left(\max_{0\le r\le n}\|\psi^{(r)}\|_\infty\right)^{-n}.
\end{equation}
Then \(c_{4,n}\) is independent of \(d\), and
\(\|D^{\boldsymbol\alpha}\phi_d\|_\infty\le c_{4,n}^{-1}\) whenever
\(|\boldsymbol\alpha|\le n\). For arbitrary coefficients
$y_1,\ldots,y_{N^d}\in\mathbb R$, define
\begin{equation}
\label{eq:bump-function}
f(\mathbf{x})=
\sum_{m=1}^{N^d}y_m
\phi_d\bigl(N(\mathbf{x}-\mathbf{x}_m)\bigr).
\end{equation}
At every point, at most one summand and its derivatives are nonzero.
Hence
\begin{equation}
\label{eq:y-bound}
\max_m|y_m|\le c_{4,n}N^{-n}
\end{equation}
implies \(f\in F_{d,n}\).

Set
\begin{equation}
\label{eq:epsilon-zero}
\epsilon_0:=\frac{c_{4,n}}{3\cdot2^n},
\qquad
N:=\left\lfloor
\left(\frac{c_{4,n}}{3\epsilon}\right)^{1/n}
\right\rfloor,
\qquad 0<\epsilon\le\epsilon_0.
\end{equation}
Then
\begin{equation}
\label{eq:N-floor-bound}
N\ge\frac12
\left(\frac{c_{4,n}}{3\epsilon}\right)^{1/n}.
\end{equation}
For an arbitrary Boolean vector
\(\mathbf z=(z_1,\ldots,z_{N^d})\), take
\(y_m=3\epsilon z_m\) in \eqref{eq:bump-function}. By the choice of
\(N\), condition \eqref{eq:y-bound} holds. If a fixed architecture
approximates every function in \(F_{d,n}\) within \(\epsilon\), write
$\eta(\mathbf x;\mathbf w)$ for its output at input $\mathbf x$ with
parameter vector $\mathbf w$. Then some parameter vector
\(\mathbf w_{\mathbf z}\) satisfies
\begin{equation}
\begin{cases}
\eta(\mathbf{x}_m;\mathbf w_{\mathbf z})>3\epsilon/2,&z_m=1,\\
\eta(\mathbf{x}_m;\mathbf w_{\mathbf z})<3\epsilon/2,&z_m=0.
\end{cases}
\end{equation}
Thus the thresholded network shatters \(S\), and
\begin{equation}
\label{eq:vcdim-lower}
\mathrm{VCdim}(\mathcal H)\ge N^d
\ge
\left(\frac{\epsilon_0}{\epsilon}\right)^{d/n}.
\end{equation}

Combining \eqref{eq:vcdim-lower} with \eqref{eq:vcdim-w2} proves
part~(a), including the dimension-uniform form
\begin{equation}
W\ge c_3^{-1/2}
\left(\frac{\epsilon_0}{\epsilon}\right)^{d/(2n)}.
\end{equation}
For fixed \(d,n,k\), decreasing the constant covers the remaining
\(\epsilon\in[\epsilon_0,1)\).

For part~(b), \eqref{eq:vcdim-depth} and
\(L\le c_1\ln^p(1/\epsilon)\) give
\begin{equation}
\label{eq:wlogw-lower}
W\ln(2W)
\ge
\frac{1}{\widetilde c_3c_1^2}
\left(\frac{\epsilon_0}{\epsilon}\right)^{d/n}
\ln^{-2p}(1/\epsilon).
\end{equation}
Exactly as in \citet{yarotsky2017error}, define the candidate scale
\(W_{c_2}:=c_2\epsilon^{-d/n}\ln^{-2p-1}(1/\epsilon)\). Then
\begin{equation}
W_{c_2}\ln(2W_{c_2})
=
\left(c_2\frac dn+o(1)\right)
\epsilon^{-d/n}\ln^{-2p}(1/\epsilon).
\end{equation}
Here $o(1)\to0$ as $\epsilon\downarrow0$.
Choosing \(c_2>0\) sufficiently small and comparing with
\eqref{eq:wlogw-lower} yields
\begin{equation}
W\ge c_2\epsilon^{-d/n}\ln^{-2p-1}(1/\epsilon).
\end{equation}
A further decrease of \(c_2\) covers all
\(\epsilon\in(0,1/2)\), proving part~(b).
\end{proof}
\subsection{Proof of the spline-KAN CoD theorem}
\label{app:proof-kan-cod}
\begin{proof}[Proof of Theorem~\ref{thm:kan_cod}]
Let \(\mathcal K\) be the spline-KAN architecture in the theorem. For every
admissible choice of its edge-wise knots,
Lemma~\ref{lem:spline-kan-repu} realizes those knot locations as affine
biases in the same \(\operatorname{ReLU}^k\)-MLP architecture. The MLP width
and depth depend on the shape, $k$, and $G$ and are independent of the knot values. Thus
the resulting MLP class contains every function represented by \(\mathcal K\),
with constants $C_k,C'_k>0$ depending only on $k$ such that
\begin{equation}
\label{eq:transfer-upper-bounds}
P_{\mathrm{MLP}}\le C_kLw^4H^2,
\qquad
P_{\mathrm{MLP}}\le C_kP_{\mathrm{KAN}}^2,
\qquad
L_{\mathrm{MLP}}\le C'_kL.
\end{equation}

If \(\mathcal K\) approximates every \(f\in F_{d,n}\) within
\(\epsilon\), so does this MLP architecture. Part~(a) of
Lemma~\ref{lem:repu_vc}, combined separately with the first two
bounds in \eqref{eq:transfer-upper-bounds}, gives
\begin{equation}
Lw^4H^2\ge c\epsilon^{-d/(2n)},
\qquad
P_{\mathrm{KAN}}\ge c'\epsilon^{-d/(4n)}.
\end{equation}
Using instead the dimension-uniform form of that theorem gives
\begin{equation}
Lw^4H^2
\ge
\bar c\left(\frac{\epsilon_0}{\epsilon}\right)^{d/(2n)},
\qquad
P_{\mathrm{KAN}}
\ge
\bar c'\left(\frac{\epsilon_0}{\epsilon}\right)^{d/(4n)},
\end{equation}
where \(\epsilon_0=\epsilon_0(n)\) and \(\bar c,\bar c'\) are independent
of \(d\).

Finally, if \(L\le c_1\ln^p(1/\epsilon)\), then
\(L_{\mathrm{MLP}}\le C'_kc_1\ln^p(1/\epsilon)\). Applying
Lemma~\ref{lem:repu_vc}(b) and again using the first two bounds in
\eqref{eq:transfer-upper-bounds} yields, respectively,
\begin{equation}
Lw^4H^2
\ge c_2\epsilon^{-d/n}\ln^{-2p-1}(1/\epsilon)
\end{equation}
and
\begin{equation}
P_{\mathrm{KAN}}
\ge c_3\epsilon^{-d/(2n)}
\ln^{-p-\frac12}(1/\epsilon).
\end{equation}
\end{proof}

\section{Experimental Details}
\label{app:experimental-details}

Parameter counts include every trainable scalar that participates in the
forward computation. Fixed grid locations, masks, and disabled branches are
excluded. In the ridge-profile experiment, each active KAN edge contributes
$G+k$ spline coefficients and one trainable spline scale. In the MNIST
experiment, the scale is fixed, so each edge contributes only its $G+k$
coefficients. Hidden depth denotes the number of nonlinear hidden stages. All
reported models omit layer-wise residual connections, and every reported KAN
is spline-only, with its edge-wise SiLU base branch set identically to zero.

\subsection{Constructive scaling protocol and source comparison}
\label{app:constructive-scaling}

Figure~\ref{fig:constructive-scaling-appendix} repeats the main-text
construction. It is not a reproduction of the trained curves in
\citet{liu2024kan}. We use
$G\in\{8,16,32,64,128\}$ intervals on every
active edge and cubic interpolation of the known scalar functions. Each
edge has $G+3$ B-spline coefficient slots. Its interpolant is converted
algebraically to a shifted ReLU$^3$ expansion, and the resulting expansions
are assembled into an ordinary dense MLP. All affine entries, including
those initialized to zero by the construction, are counted for the MLP.
Consequently, the horizontal coordinates of the two families use stored
trainable slots, but they do not represent a matched-parameter comparison.

The one-dimensional Bessel target uses a $[1,1]$ KAN on $[-1,1]$.
The target $e^{\sin(\pi x)+y^2}$ uses $[2,1,1]$ on $[-1,1]^2$.
For the 100-dimensional target we use
$\exp\!\bigl(100^{-1}\sum_{i=1}^{100}\sin^2(\pi x_i/2)\bigr)$ and a
$[100,1,1]$ construction on $[0,1]^{100}$. The input domain for this
target is not specified in \citet{liu2024kan}. For the four-dimensional
nested target, we follow the formula shown there,
$\exp\!\bigl(\sin(x_1^2+x_2^2)+\sin(x_3^2+x_4^2)\bigr)$, on $[-1,1]^4$.
Our explicit construction has shape $[4,2,1,1]$. The paper's prose,
rendered formula, and plotted topology are not fully consistent. We omit
$xy$ from the main figure because
$xy=((x+y)^2-(x-y)^2)/4$ is represented exactly by cubic-spline edges,
leaving only floating-point error.

We use independently scrambled Sobol sets of 1,024 training, 2,048
validation, and 8,192 test points, fixed across grids. The transferred
MLPs are continued in float64 with Adam for at most 200 steps, minibatches
of 256, and learning rate $10^{-6}(8/G)^4$. The validation set selects
among steps $0,1,2,5,10,20,50,100,150,200$, including the unmodified
transfer at step zero. The independent
test set is used after selection. All computations run on a server CPU with the
edge-wise SiLU branch absent. This one-seed protocol does not estimate
training variability. For the 100-dimensional and four-dimensional targets,
validation selects step zero at every grid, so their purple curves show
unmodified transfers.

For the original trained KAN/MLP scaling-law curves, see
\citet{liu2024kan}. Their optimization protocol and parameter-count
convention differ from ours, so the figures do not form a matched comparison.

\begin{figure}[H]
  \centering
  \includegraphics[width=0.98\linewidth]{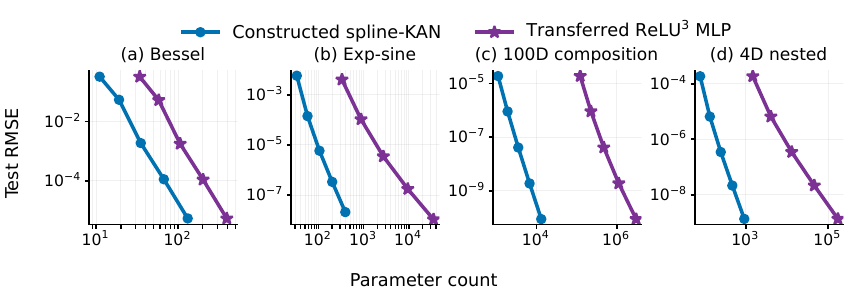}
  \caption{Constructive scaling on four targets from \citet{liu2024kan}.
  Blue: constructed cubic spline-KAN test RMSE versus spline-coefficient
  count. Purple: transferred ReLU$^3$ MLP test RMSE versus dense parameter
  count, at the validation-selected checkpoint (step zero or up to 200 Adam
  steps). At step zero the MLP represents the same function as its KAN source.
  All panels use $G\in\{8,16,32,64,128\}$.}
  \label{fig:constructive-scaling-appendix}
\end{figure}

\subsection{Controlled ridge-profile protocol}

The targets are defined in \eqref{eq:ridge-profile-teacher}. Before normalization,
the coefficients and direction entries are sampled independently from the
standard normal distribution $\mathcal N(0,1)$. For each target instance, the coefficient vector
$\boldsymbol a=(a_1,\ldots,a_R)$ and every direction $\boldsymbol v_r$ are
normalized in $\ell_2$. The support of each $\boldsymbol v_r$
is sampled uniformly subject to $\|\boldsymbol v_r\|_0=s$, where
$\|\boldsymbol v\|_0$ denotes the number of nonzero coordinates of
$\boldsymbol v$, and each offset $b_r$ is sampled uniformly from
$[-0.5,0.5]$. The mean $\mu_{\varphi}$ and standard
deviation $\sigma_{\varphi}$ are estimated from 16,384 independently sampled reference
points. Training uses fresh online samples at every step. Fixed evaluation
sets contain 4,096 training-probe points and 16,384 validation and test points
each.

\begin{table}[H]
\caption{Controlled ridge-profile scan.}
\label{tab:ridge-profile-hparams}
\centering
\small
\begin{tabular}{@{}>{\raggedright\arraybackslash}p{0.24\linewidth}>{\raggedright\arraybackslash}p{0.70\linewidth}@{}}
\toprule
Item & Setting \\
\midrule
Target grid & $d,R\in\{4,16,64\}$, and
$s\in\{1,2,4,\ldots,d\}$ \\
Target profiles & $\varphi(t)=t^2$, $\sin(2\pi t)$, $\tanh(t)$, or
$\operatorname{SiLU}(t)=t/(1+e^{-t})$ \\
Hidden depth & $L\in\{2,4,6\}$ for both families \\
MLP & Width $w\in\{1,4,16,64,256\}$; ReLU, tanh, SiLU, or sine
activation \\
Spline-only KAN & Width $w\in\{1,4,16,64\}$; $k=3$;
$G\in\{3,7,15,31\}$; zero edge-wise base branch \\
Optimization & MSE; Adam; zero weight decay; 5,000 online-sampling steps;
batch size 4,096; 500-step linear warmup followed by cosine decay to $1\%$ of
the base rate \\
Learning rates & $10^{-3}$, $3\times10^{-3}$, and $10^{-2}$ for both
families \\
Seeds & Seed index $q\in\{0,1,2\}$ \\
Budgets & $P\in\{250,500,1000,2000,5000,10^4,3\times10^4,10^5\}$ \\
Selection & Within each target cell, model family, and budget, select the run
with minimum training-probe MSE over architectures, learning rates, and seeds;
report both its training-probe MSE and its independent test MSE \\
\bottomrule
\end{tabular}
\end{table}

Within each run, validation MSE evaluated every 100 steps selects the
checkpoint. The fixed training probe is reserved for constructing the final
envelope. Pooling the three seeds is deliberate because the experiment asks
for the best fit found within a fixed search budget. This convention is
applied symmetrically to both families. For each profile, the target cells
form a complete Cartesian product under the scan in
Table~\ref{tab:ridge-profile-hparams}. There are $45$ target cells per
profile, $324$ architecture--learning-rate configurations per cell and seed,
and $174{,}960$ successful runs across all four profiles and three seeds.
No manifest jobs are missing or failed.

\subsection{Target profile and activation sensitivity}
\label{app:ridge-robustness}

Figure~\ref{fig:ridge-square-full-appendix} shows the complete square-profile
scan. Its $d=16$ row is Figure~\ref{fig:ridge-square-boundary} in the main text.
The main-text envelope allows the MLP to select among four activations. This
provides a strong comparator, but the sine target makes activation alignment
especially relevant. Figures~\ref{fig:ridge-sine-appendix}--
\ref{fig:ridge-tanh-appendix} show the three other target profiles under the
same scan and selection rule as Figure~\ref{fig:ridge-square-boundary}.
All four profiles exhibit a strong KAN regime at $s=1$, while the crossover
at larger $s$ depends on the target and parameter budget. In particular,
square and SiLU do not support a universal MLP win for interacting targets.
For sine at $d=64$, some dense cells leave both models near the standardized
unit-variance baseline. Their near-unit error ratios are not evidence of
successful fitting by either family.

Figure~\ref{fig:sine-ridge-nonsine} removes
the target-matched sine activation and lets the MLP choose among ReLU, tanh,
and SiLU. KAN then has lower test error in more cells, including 20 of 24 at
$P=10^4$ within the $d\in\{4,16\}$ subset, but the gap between $s=1$ and
denser supports remains large.
Figure~\ref{fig:sine-ridge-relu} fixes ReLU and shows the same directional
effect. The erosion of KAN's sparse-target advantage therefore persists with
a ReLU comparator, although target--activation matching substantially changes
where the crossover occurs.
Figure~\ref{fig:sine-ridge-train-probe} reports training-probe ratios for
the same selected models.

\begin{figure*}[p]
  \centering
  \includegraphics[width=0.98\textwidth]{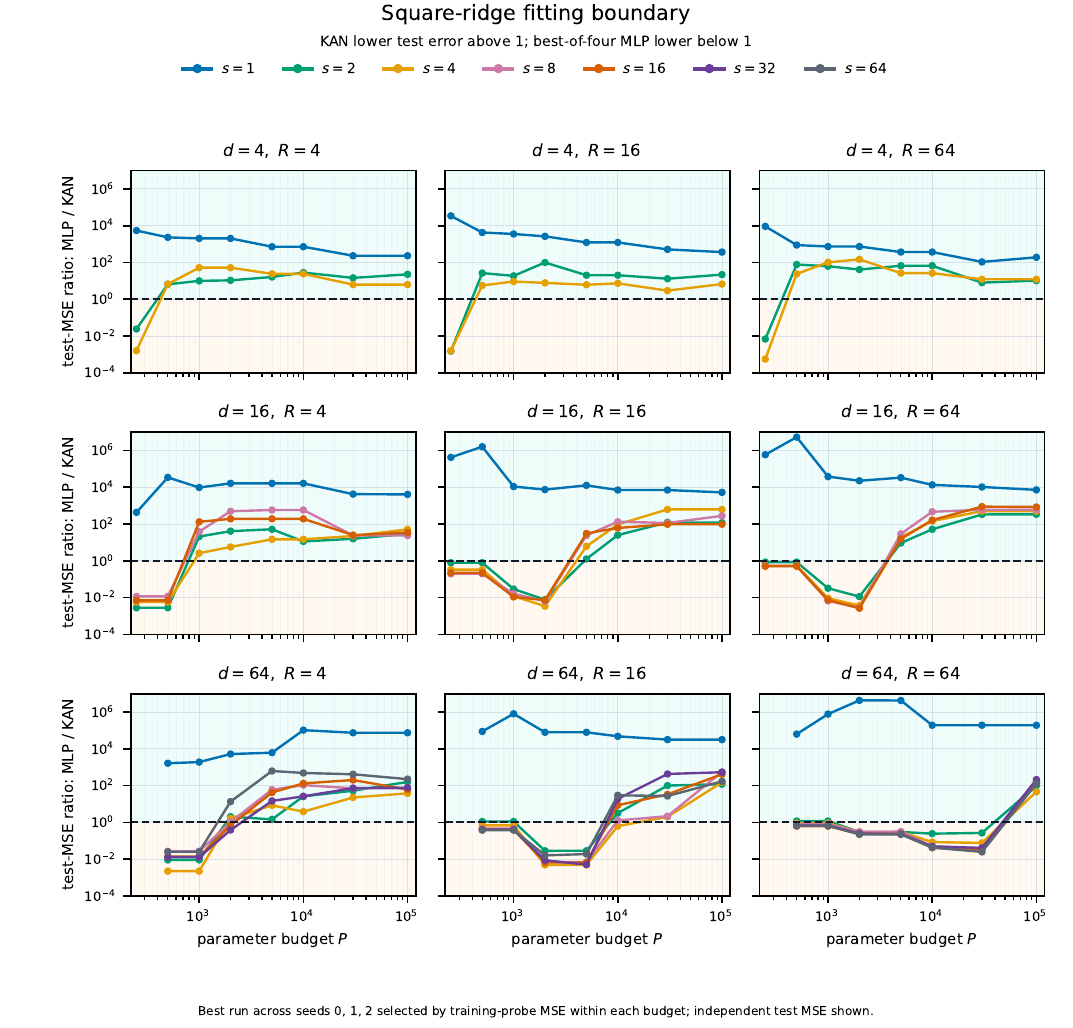}
  \caption{Complete square-ridge test-MSE ratios for
  $d,R\in\{4,16,64\}$, under the same training-probe selection rule as
  Figure~\ref{fig:ridge-square-boundary}. The middle row is shown in the
  main text.}
  \label{fig:ridge-square-full-appendix}
\end{figure*}

\begin{figure*}[p]
  \centering
  \includegraphics[width=0.98\textwidth]{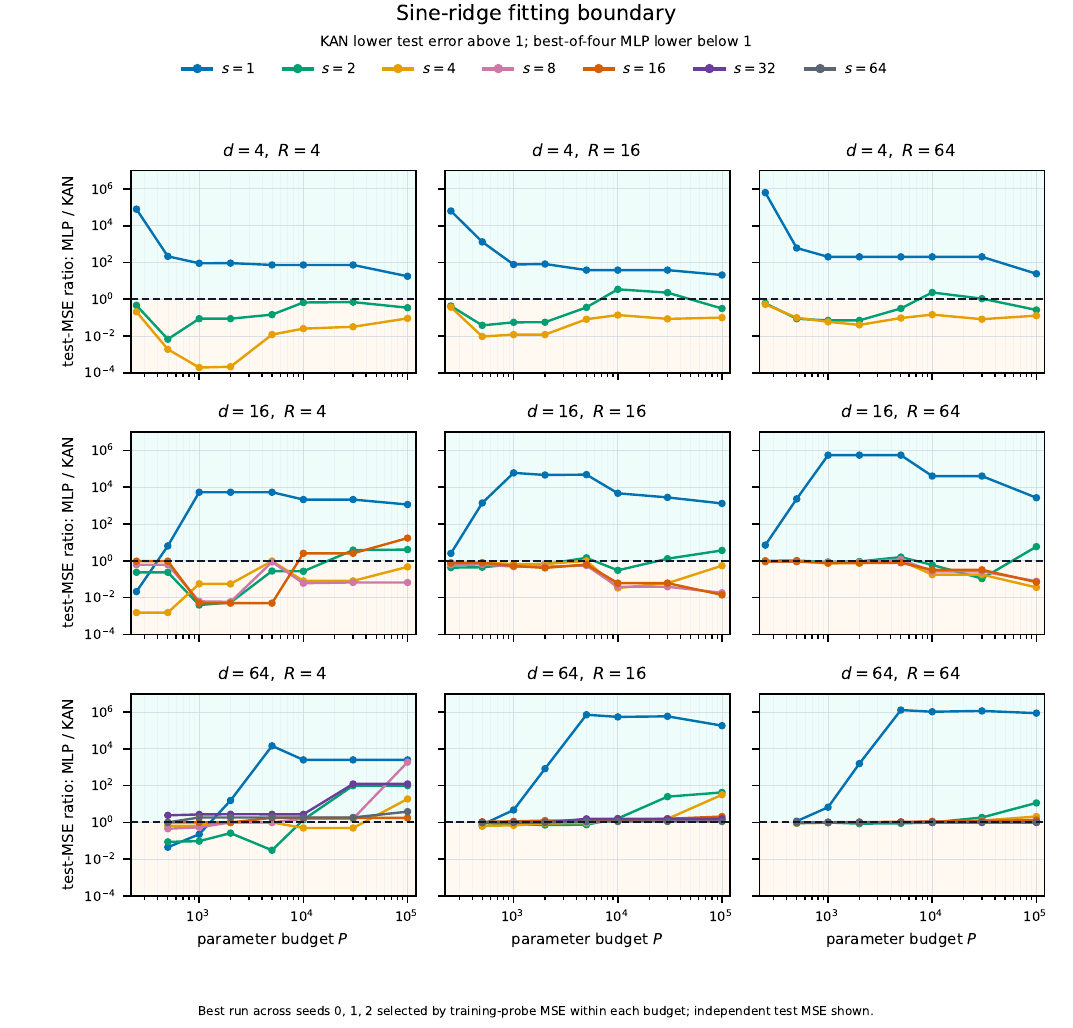}
  \caption{Sine-ridge test-MSE ratios under the same three-seed scan and
  parameter-budget selection as Figure~\ref{fig:ridge-square-boundary}.}
  \label{fig:ridge-sine-appendix}
\end{figure*}

\begin{figure*}[p]
  \centering
  \includegraphics[width=0.98\textwidth]{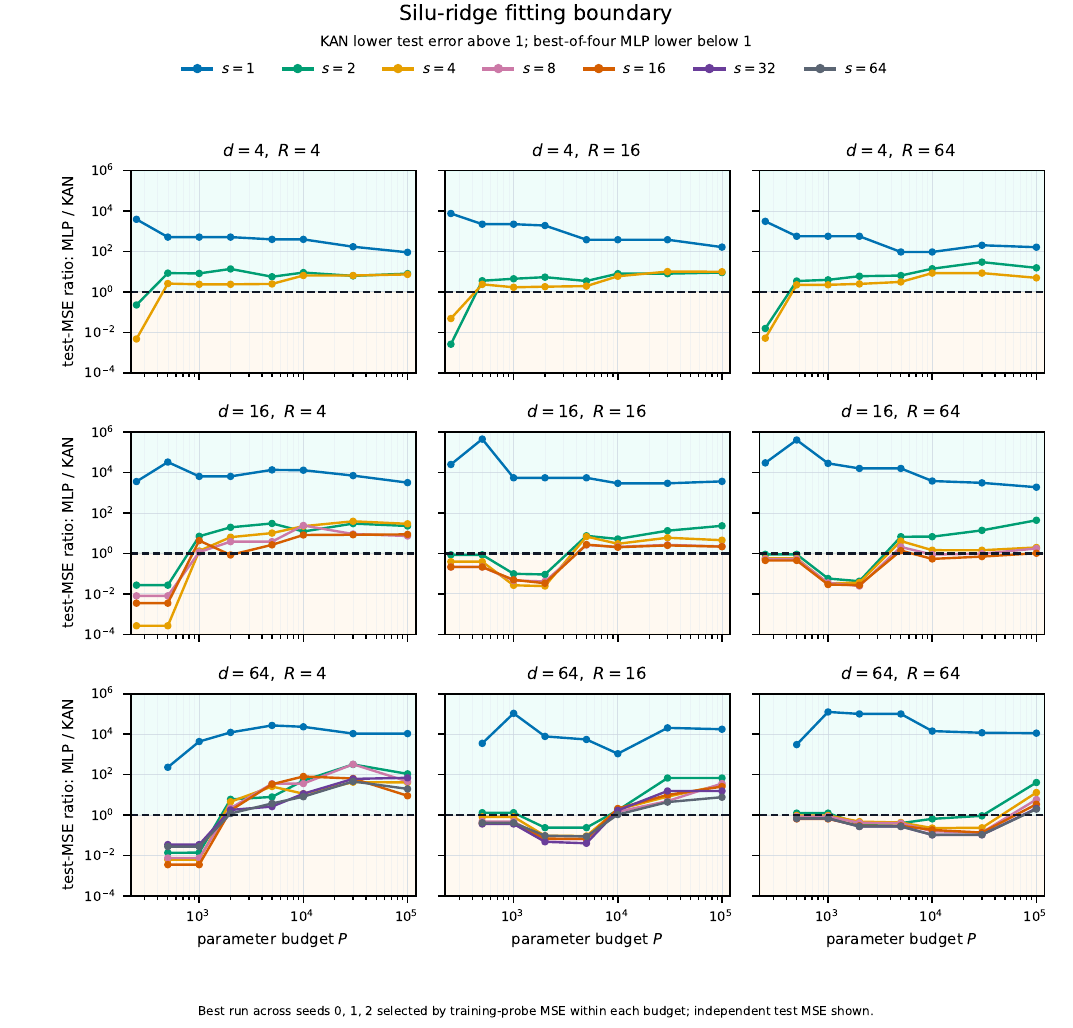}
  \caption{SiLU-ridge test-MSE ratios under the same scan and selection rule.}
  \label{fig:ridge-silu-appendix}
\end{figure*}

\begin{figure*}[p]
  \centering
  \includegraphics[width=0.98\textwidth]{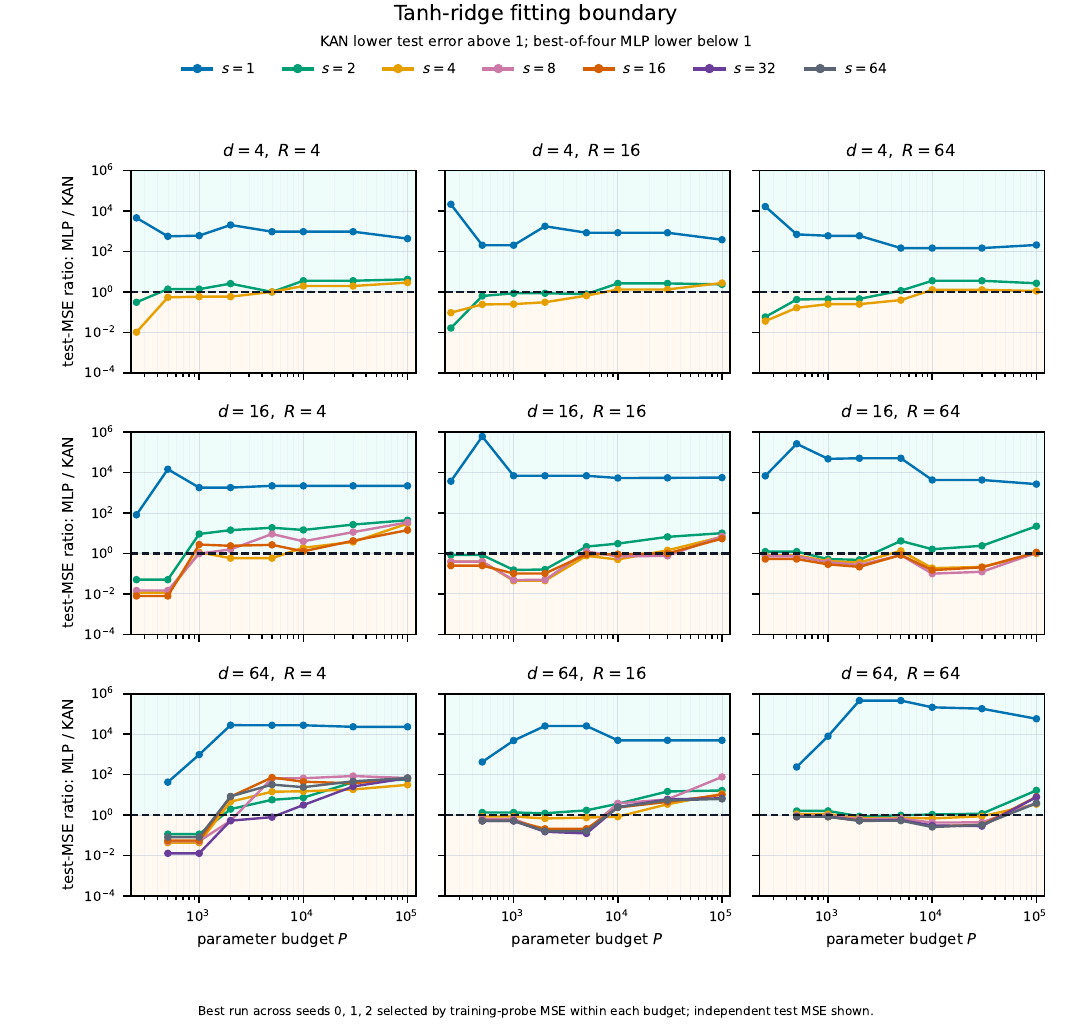}
  \caption{Tanh-ridge test-MSE ratios under the same scan and selection rule.}
  \label{fig:ridge-tanh-appendix}
\end{figure*}

\subsubsection{Exact representation and optimization diagnostic}
\label{app:sine-optimization-diagnostic}

For the sine profile in \eqref{eq:ridge-profile-teacher}, the target has the
exact one-hidden-layer sine-MLP representation
\begin{equation}
  f_{\sin}(\mathbf x)=\gamma_0+\sum_{r=1}^{R}\gamma_r
  \sin(\boldsymbol\omega_r^\top\mathbf x+\beta_r).
  \label{eq:sine-ridge-exact-mlp}
\end{equation}
For completeness, the coefficients in
\eqref{eq:sine-ridge-exact-mlp} can be chosen explicitly. Writing
$\mathbf 1$ for the all-ones vector and using the normalization in
\eqref{eq:ridge-profile-teacher}, set
\begin{equation}
  \gamma_0=-\frac{\mu_{\sin}}{\sigma_{\sin}},\qquad
  \gamma_r=\frac{a_r}{\sigma_{\sin}},\qquad
  \boldsymbol\omega_r=4\pi\boldsymbol v_r,\qquad
  \beta_r=2\pi\bigl(b_r-\boldsymbol v_r^\top\mathbf 1\bigr).
\end{equation}
This yields the target exactly. A dense one-hidden-layer implementation has
$dR$ input weights, $R$ hidden biases, $R$ output weights, and one output
bias, hence $R(d+2)+1$ trainable parameters.

We verify the distinction between representation and optimization at
$d=R=64$ using a width-$64$ sine MLP and the same Adam, warmup--cosine
schedule, batch size, and 5,000-step protocol as the main scan. The explicit
weights give test MSE $1.96\times10^{-13}$ without training. From the default
SIREN-style random initialization, the $s=1$ model finishes at test MSE
$0.542$. Supplying only the teacher ridge directions, while leaving phases and
output coefficients to be learned, reduces the test MSE to
$8.90\times10^{-6}$ for $s=1$, $4.42\times10^{-7}$ for $s=16$, and
$3.41\times10^{-7}$ for $s=64$. The high error persists when only the
first-layer weight norm is matched while its directions remain random.
Because direction-aligned initialization uses knowledge of the target, we
reserve these values for the optimization diagnostic. The exact construction
and initialization comparison identify optimization and feature discovery as
the source of the trained error.

\begin{figure*}[htbp]
  \centering
  \includegraphics[width=0.90\textwidth]{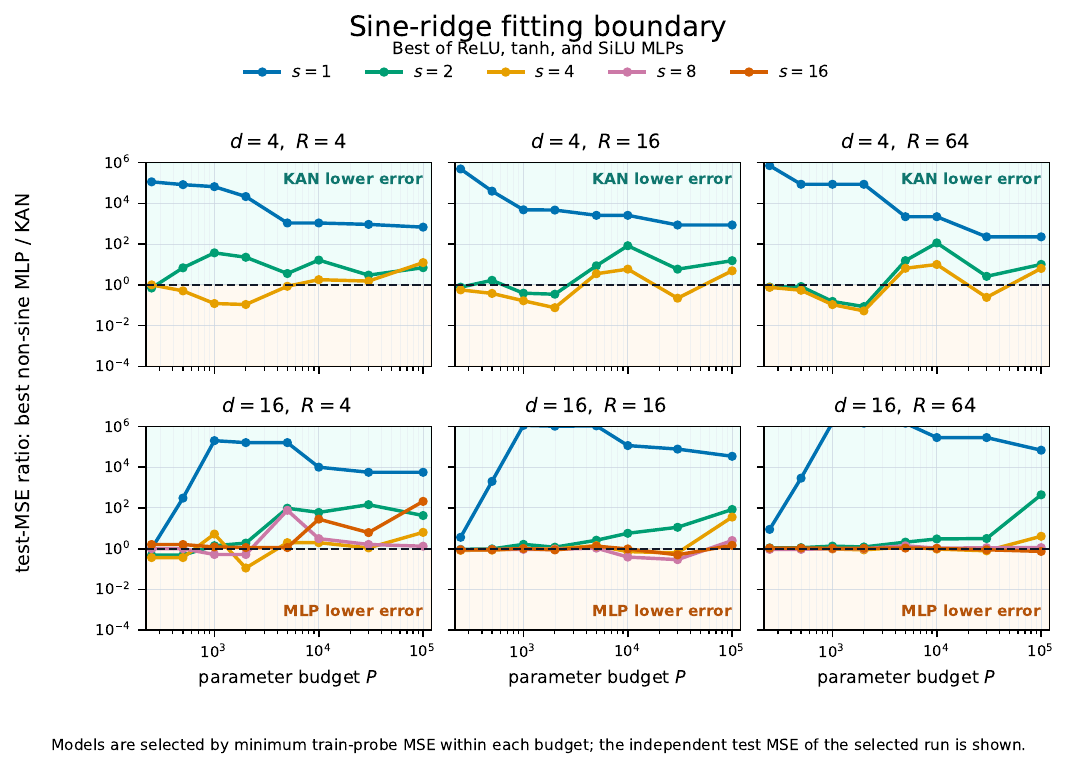}
  \caption{Sine-ridge test-error ratios when the MLP envelope is restricted to
  ReLU, tanh, or SiLU. Plotting and selection conventions follow
  Figure~\ref{fig:ridge-sine-appendix} for the $d\in\{4,16\}$ subset.
  Removing sine activation expands the KAN-favored region, while denser
  supports still narrow that advantage.}
  \label{fig:sine-ridge-nonsine}
\end{figure*}

\begin{figure*}[htbp]
  \centering
  \includegraphics[width=0.90\textwidth]{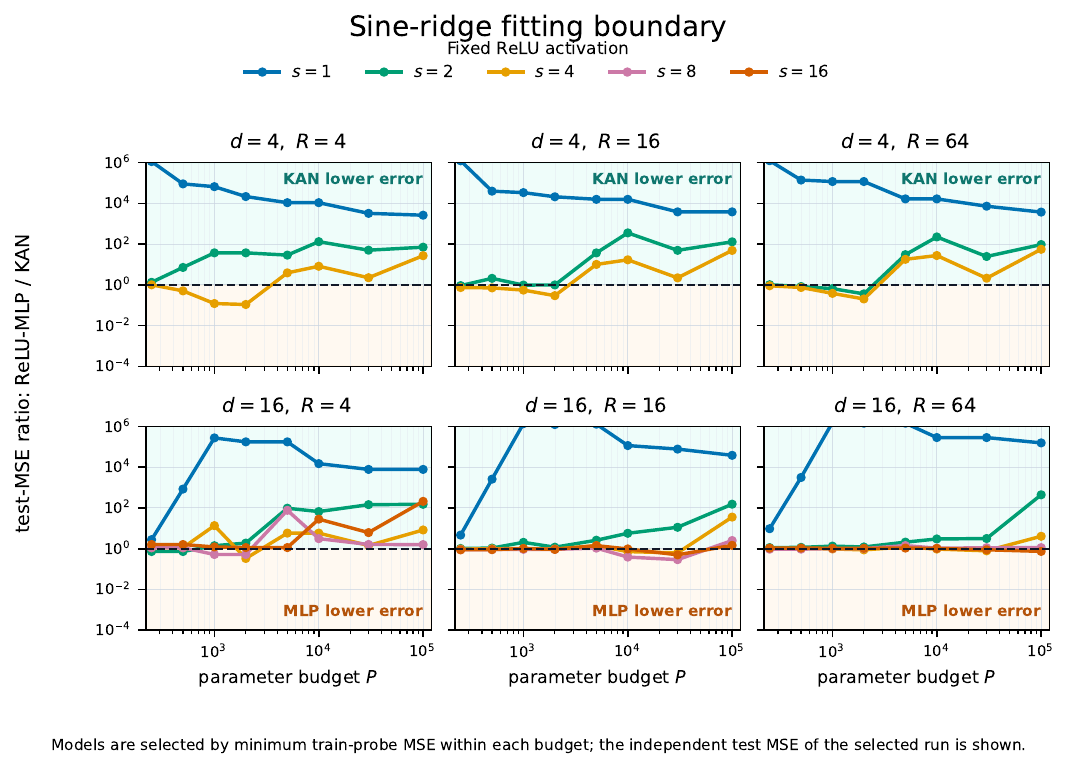}
  \caption{Sine-ridge test-error ratios for a fixed ReLU-MLP comparator.
  Increasing support $s$ again shifts the comparison toward MLP with a ReLU
  comparator.}
  \label{fig:sine-ridge-relu}
\end{figure*}

\begin{figure*}[htbp]
  \centering
  \includegraphics[width=0.90\textwidth]{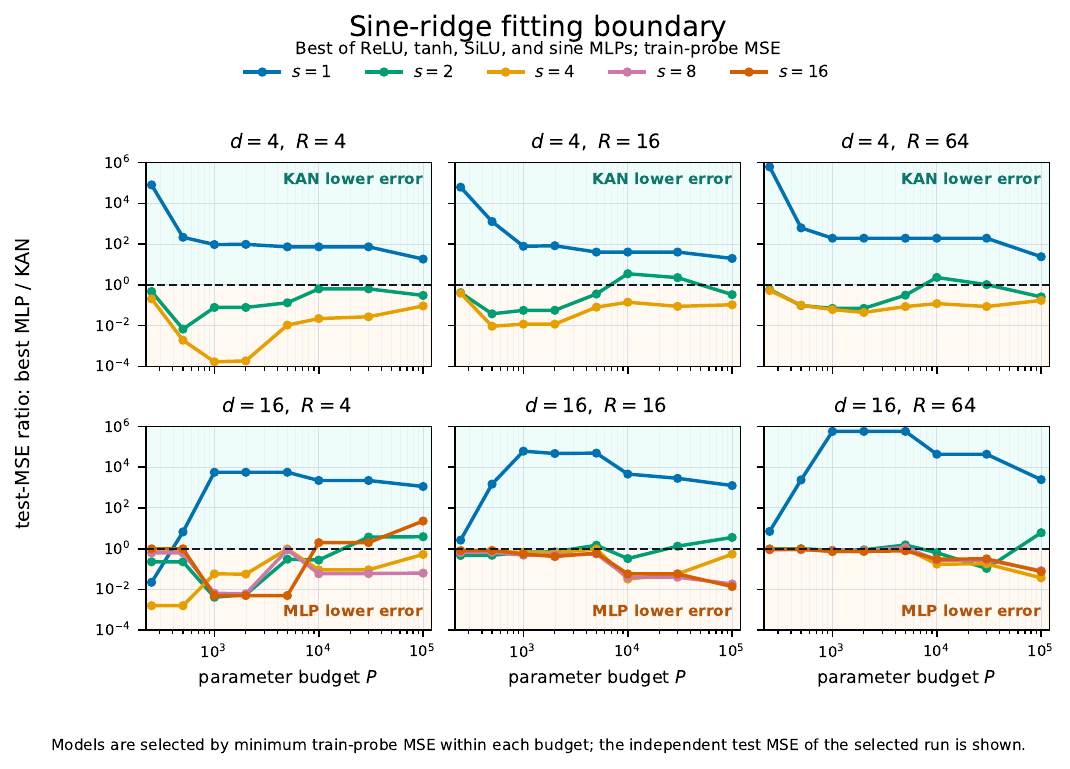}
  \caption{Training-probe MSE ratios for the best-activation MLP comparison.
  The panels cover $d\in\{4,16\}$ with the same selection rule as
  Figure~\ref{fig:ridge-sine-appendix}. The support-dependent pattern matches
  the independent-test ratios.}
  \label{fig:sine-ridge-train-probe}
\end{figure*}

Because a large error ratio can be misleading when both models fit poorly,
Figure~\ref{fig:sine-ridge-absolute-mse} reports the corresponding absolute
errors. The targets are standardized to unit variance, so the constant-zero
predictor has MSE approximately one. At moderate and large parameter budgets,
most selected models lie well below this baseline, and the training-probe and
independent-test curves closely track one another. Thus the ratios generally
do not compare two models stuck near the trivial predictor.

\begin{figure*}[htbp]
  \centering
  \includegraphics[width=0.98\textwidth]{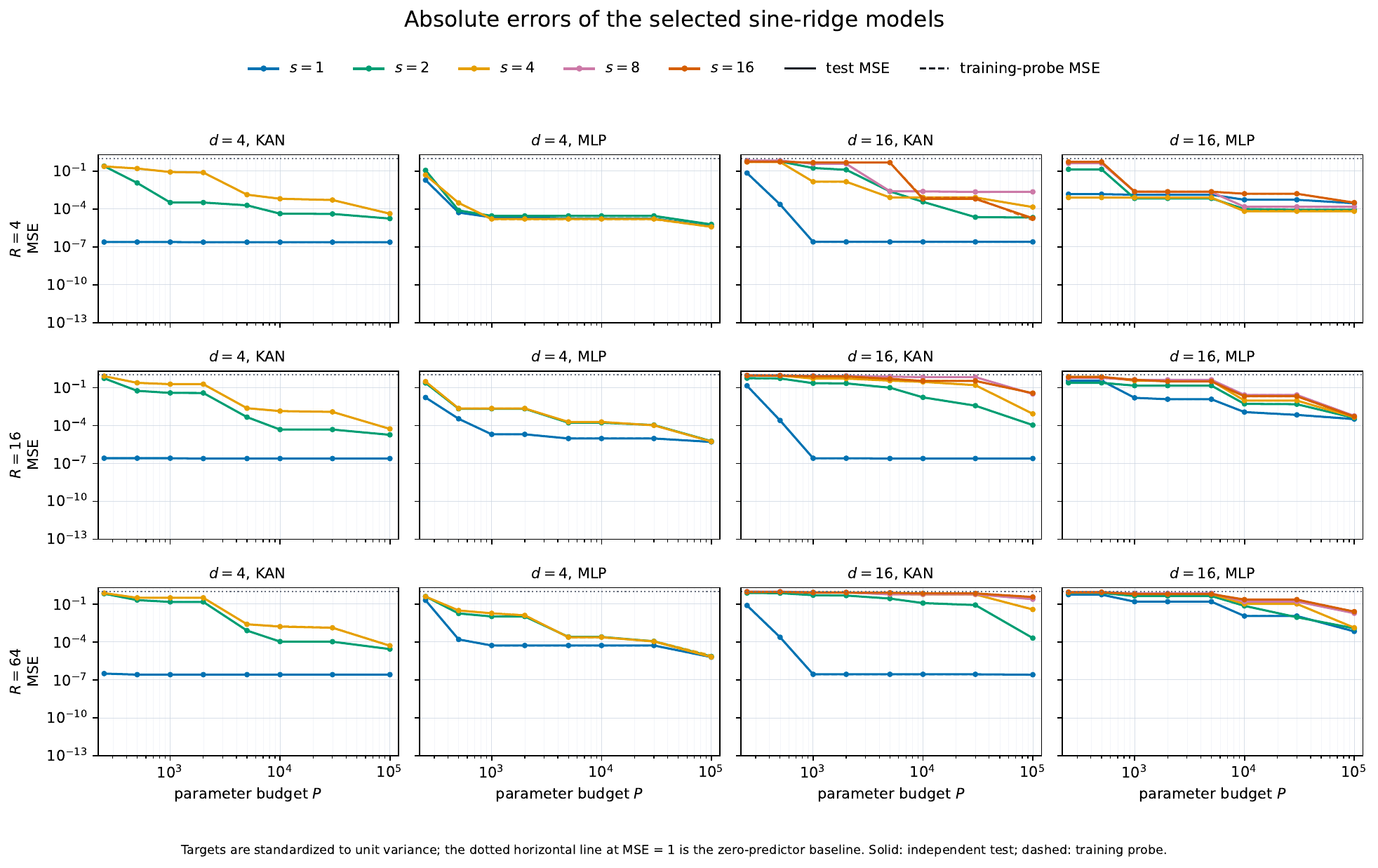}
  \caption{Absolute MSE of the models selected for the best-activation
  sine-ridge envelope. Solid curves report independent-test MSE and dashed
  curves report training-probe MSE. The dotted horizontal line marks MSE one,
  the zero-predictor baseline after target standardization.}
  \label{fig:sine-ridge-absolute-mse}
\end{figure*}

\clearpage
\subsection{MNIST classification and random-label test}
\label{app:mnist-experiment}

For each seed $q\in\{0,1,2\}$, all 60,000 images in the official MNIST
training partition are flattened to 784 coordinates in $[0,1]$. In the
normal-label condition, the official labels are retained and evaluation uses
all 10,000 images in the official test partition. In the random-label
condition, each training image $i=1,\ldots,60{,}000$ receives an independent
label $y_i$ sampled uniformly from $\{0,\ldots,9\}$. Within each condition and seed,
both model families receive the same images and labels.

\begin{table}[ht]
\caption{MNIST classification and random-label fitting hyperparameters.}
\label{tab:mnist-hparams}
\centering
\small
\begin{tabular}{@{}>{\raggedright\arraybackslash}p{0.24\linewidth}>{\raggedright\arraybackslash}p{0.70\linewidth}@{}}
\toprule
Item & Setting \\
\midrule
Hidden depth & $L\in\{2,4,6,8\}$ \\
Hidden width & MLP: $w\in\{4,8,16,32,64,128,256,512\}$; KAN:
$w\in\{1,2,4,8,16,32\}$ \\
MLP & ReLU activation; Kaiming-uniform weights; zero biases \\
Spline-only KAN & $G=3$, $k=3$; zero edge-wise base branch; spline
coefficients trainable and spline scale fixed \\
Optimization & Cross-entropy, batch size 256, 200 epochs, no weight decay;
SGD with momentum $0.9$ and learning rate in $\{0.1,0.03,0.01\}$, or Adam
with learning rate in $\{0.01,0.003,0.001\}$ \\
Aggregation & For each architecture and seed, retain the optimizer/rate with
highest final training accuracy; report medians and interquartile ranges over
three seeds \\
\bottomrule
\end{tabular}
\end{table}

For normal labels, the final training and test accuracies are read from the
optimizer/rate selected by final training accuracy, so the test frontier forms
a descriptive envelope over the prescribed scan. The independently randomized
label condition reports training accuracy as its measure of random-label fitting.

Table~\ref{tab:mnist-normal-random} in the main text summarizes selected
budgets. Figure~\ref{fig:mnist-normal-random-frontiers-appendix} shows the full
budget frontiers, while Figure~\ref{fig:mnist-architecture-heatmaps} exposes
the complete depth--width scan. Over the scanned budgets, the MLP test envelope on normal
labels rises faster and reaches a higher accuracy. With random labels, wider
MLPs approach 100\% training accuracy, while the scanned KANs remain below
20\%.

\begin{figure}[ht]
  \centering
  \includegraphics[width=0.88\linewidth]{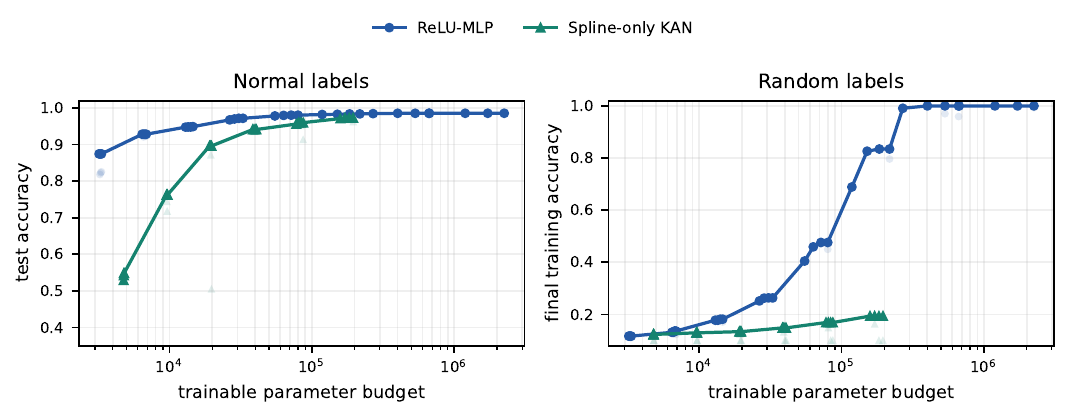}
  \caption{Full MNIST accuracy frontiers with $N_{\mathrm{train}}=60{,}000$.
  The left panel reports standard-label test accuracy and the right panel
  reports random-label final training accuracy. Each point is a three-seed
  architecture median after optimizer and learning-rate selection by final
  training accuracy. The curves show the best available value within each
  parameter budget.}
  \label{fig:mnist-normal-random-frontiers-appendix}
\end{figure}

\begin{figure}[ht]
  \centering
  \includegraphics[width=\linewidth]{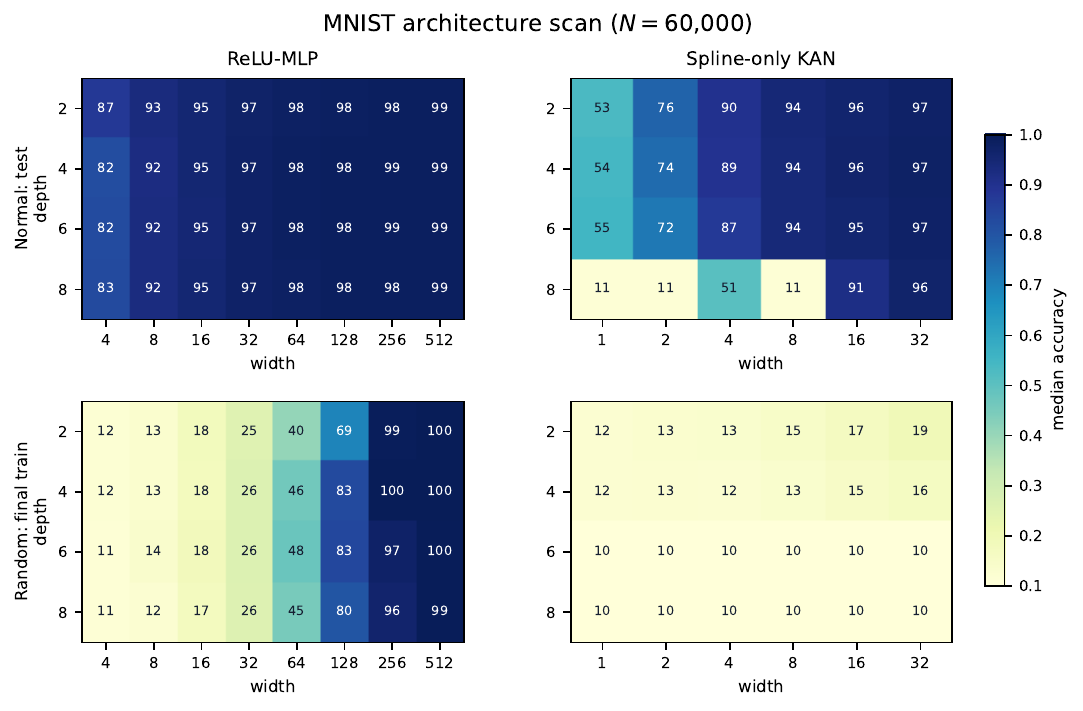}
  \caption{Depth--width scan for normal-label testing (top) and random-label
  fitting (bottom). Cell values are median accuracies in percent. Width
  grids differ between MLP and KAN so that the scan covers overlapping
  parameter budgets. The plot title's $N=60{,}000$ denotes the number of
  training examples. The random-label fitting gap persists across the
  tested depths.}
  \label{fig:mnist-architecture-heatmaps}
\end{figure}

As a small-sample diagnostic, we repeat both conditions using a fixed subset
of 6,000 training images. This scan also includes MLP widths $4$ and $8$, so
the two families overlap below $10^4$ trainable parameters.
Figure~\ref{fig:mnist-n6000-frontiers} shows the corresponding frontiers.
Fitting random labels is easier at this sample size, but the qualitative
separation persists: the MLP eventually fits all training labels, whereas the
best spline-only KAN plateaus below 100\% training accuracy.

\begin{figure}[ht]
  \centering
  \includegraphics[width=\linewidth]{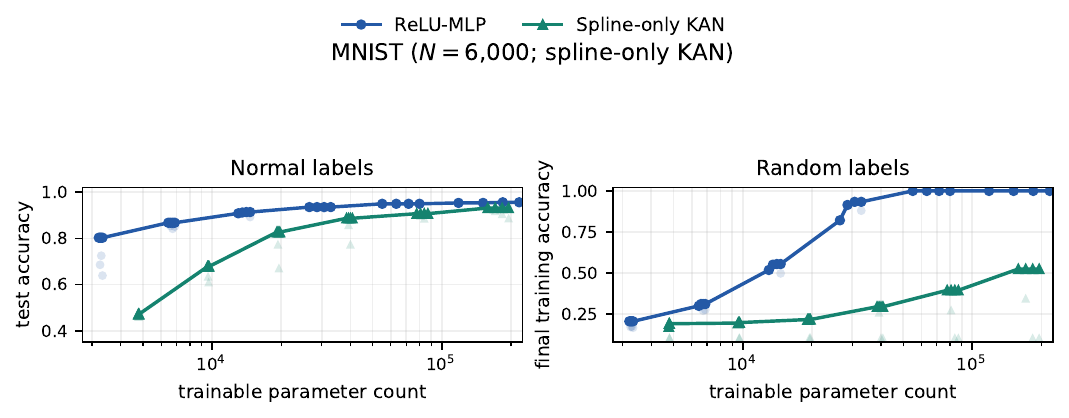}
  \caption{Small-sample MNIST comparison with $N_{\mathrm{train}}=6{,}000$.
  The normal-label
  panel reports median test accuracy and the random-label panel reports median
  training accuracy over three seeds. The display is truncated just beyond
  the largest scanned KAN budget so that the visible range is covered by both
  families.}
  \label{fig:mnist-n6000-frontiers}
\end{figure}

\end{document}